\documentclass[runningheads,a4paper]{llncs}

\usepackage{subcaption}
\usepackage{xcolor}
\usepackage{hyperref}
\usepackage{tikz}
\usetikzlibrary{arrows,shapes}
\usetikzlibrary{calc}

\usepackage{enumitem}

\usepackage{amssymb}
\usepackage{graphicx}

\usepackage{stackrel}
\usepackage{mathtools}

\usepackage[dvips]{epsfig}
\usepackage[english,francais]{babel}
\usepackage[latin1]{inputenc}

\usepackage{url}
\urldef{\mailsa}\path|{alfred.hofmann, ursula.barth, ingrid.haas,
frank.holzwarth,| \urldef{\mailsb}\path|anna.kramer, leonie.kunz,
christine.reiss, nicole.sator,|
\urldef{\mailsc}\path|erika.siebert-cole, peter.strasser,
lncs}@springer.com|

\begin{document}

\mainmatter  %
\title{Morse sequences
}

\titlerunning{Morse sequences}

\author{Gilles Bertrand\orcidID{0009-0004-7294-7081}}

\authorrunning{G. Bertrand}

\institute{Univ Gustave Eiffel, CNRS, LIGM, F-77454 Marne-la-Vall\'{e}e, France
\email{gilles.bertrand@esiee.fr}}

\maketitle

\newcommand{\bbbu}{\; \ddot{\cup} \;}
\newcommand{\axr}[1]{\ddot{\textsc{#1}}  \normalsize}

\newcommand{\axcup}{\textsc{C\tiny{UP}} }
\newcommand{\axcap}{\textsc{C\tiny{AP}} }
\newcommand{\axunion}{\textsc{U\tiny{NION}} }
\newcommand{\axinter}{\textsc{I\tiny{NTER}} }

\newcommand{\bb}[1]{\mathbb{#1}}
\newcommand{\ca}[1]{\mathcal{#1}}
\newcommand{\ax}[1]{\textsc{#1} \normalsize}

\newcommand{\axb}[2]{\ddot{\textsc{#1}} \textsc{\tiny{#2}}  \normalsize}
\newcommand{\bbb}[1]{\ddot{\mathbb{#1}}}
\newcommand{\cab}[1]{\ddot{\mathcal{#1}}}

\newcommand{\rel}[1]{\scriptstyle{\mathbf{#1}}}

\newcommand{\rela}[1]{\textsc{\scriptsize{\bf{{#1}}}} \normalsize}

\newcommand{\de}[2]{#1[#2]}
\newcommand{\di}[2]{#1\langle #2 \rangle}

\newcommand{\la}{\langle}
\newcommand{\ra}{\rangle}
\newcommand{\hs}{\hspace*{\fill}}

\newcommand{\cell}{\mathbb{C}}
\newcommand{\cellp}{\mathbb{C}^\times}
\newcommand{\simp}{\mathbb{S}}
\newcommand{\comp}{\mathbb{H}}
\newcommand{\simpp}{\mathbb{S}^\times}
\newcommand{\den}{\mathbb{D}\mathrm{en}}
\newcommand{\ram}{\mathbb{R}\mathrm{am}}
\newcommand{\tree}{\mathbb{T}\mathrm{ree}}
\newcommand{\graph}{\mathbb{G}\mathrm{raph}}
\newcommand{\vertex}{\mathbb{V}\mathrm{ert}}
\newcommand{\edge}{\mathbb{E}\mathrm{dge}}
\newcommand{\equ}{\mathbb{E}\mathrm{qu}}
\newcommand{\esub}{\mathbb{E}\mathrm{sub}}

\newcommand{\topp}{\langle \mathrm{K} \rangle}
\newcommand{\topq}{\langle \mathrm{Q} \rangle}
\newcommand{\topxp}{\langle \mathbb{X}, \mathrm{P} \rangle}
\newcommand{\topxq}{\langle \mathbb{X},\mathrm{Q} \rangle}

\newcommand{\vt}{\mathcal{K}}
\newcommand{\vtp}{\mathcal{T'}}
\newcommand{\vtpp}{\mathcal{T''}}
\newcommand{\vq}{\mathcal{Q}}
\newcommand{\vqp}{\mathcal{Q'}}
\newcommand{\vqpp}{\mathcal{Q''}}
\newcommand{\vk}{\mathcal{K}}
\newcommand{\vkp}{\mathcal{K'}}
\newcommand{\vkpp}{\mathcal{K''}}

\newcommand{\C}{\ensuremath{\searrow^{\!\!\!\!\!C}}}
\newcommand{\Detach}{\ensuremath{\;\oslash\;}}

\newcommand{\sig}{\sigma}
\newcommand{\del}{W}
\newcommand{\ms}{\overrightarrow{W}}
\newcommand{\msi}{\overrightarrow{W_{i}}}
\newcommand{\msim}{\overrightarrow{W_{i-1}}}
\newcommand{\mss}{\widehat{W}}
\newcommand{\msc}{\ddot{W}}
\newcommand{\mso}{\overleftarrow{W}}

\begin{abstract}
We introduce the notion of a Morse sequence, which provides a simple and effective approach to discrete Morse theory.
A Morse sequence is a sequence composed solely of two elementary operations, that is, 
expansions (the inverse of a collapse), and fillings (the inverse of a perforation).
We show that a Morse sequence may be
seen as an alternative way to represent the gradient vector field of an arbitrary
discrete Morse function. We also show that it is possible, in a straightforward
manner, to make a link between Morse sequences and different kinds of Morse
functions. At last, we introduce maximal Morse sequences, which formalize 
two basic schemes for building a Morse sequence from
an arbitrary simplicial complex.

\end{abstract}

\section{Introduction}

Discrete Morse theory, developed by Robin Forman  \cite{For98,For02}, studies the topology of objects 
using functions that assign values to their cells of different dimensions.
A discrete Morse function detects some special cells, 
called critical cells, which capture the essential topological features of the object. 

In this paper, we present an approach where, instead of a Morse function, a sequence of elementary operators 
is used for a simple representation of an object. This sequence, that we called \emph{a Morse sequence},
is composed solely of two elementary operations, that is,
expansions (the inverse of a collapse), and fillings (the inverse of a perforation).
These operations correspond exactly to the ones introduced by Henry Whitehead \cite{Whi39}.
After some basic definitions and two meaningful examples (Sections \ref{sec:basic}, \ref{sec:sequences}, and \ref{sec:optimal}), we show that a Morse sequence is
an alternative way to represent the gradient vector field of an arbitrary
discrete Morse function (Section \ref{sec:gradient}). We also show that it is possible to recover immediately, from a Morse sequence,
different kinds of Morse
functions (Section \ref{sec:functions}). At last, we introduce maximal Morse sequences, which formalize
two basic schemes for building a Morse sequence from
an arbitrary simplicial complex (Section \ref{sec:maximal}).

\newpage

\section{Basic definitions}
\label{sec:basic}

Let $K$ be a finite family composed
of non-empty finite sets.
The family $K$ is a {\it (simplicial) complex} if $\sigma \in K$ whenever $\sigma \not= \emptyset$ and $\sigma \subseteq \tau$
for some $\tau \in K$.

An element of a simplicial complex $K$ is {\it a simplex of $K$} or {\it a face of $K$}.
A {\em facet of $K$} is a simplex of $K$ that is maximal for inclusion.
The {\it dimension} of $\sigma \in K$, written $dim(\sigma)$,
is the number of its elements
minus one. If $dim(\sigma) =p$, we say that $\sigma$ is a \emph{$p$-simplex}.
The {\it dimension of $K$}, written $dim(K)$,
is the largest dimension of its simplices,
the {\it dimension of $\emptyset$}, the void complex,  being defined to be $-1$.
We denote by $K^{(p)}$ the set composed of all $p$-simplexes of $K$.

If $\sigma \in K^{(p)}$ we set
$\partial(\sigma) = \{\tau \in K^{(p-1)} \; | \; \tau \subset \sigma \}$,
which is the {\em boundary of $\sigma$}.

We recall some basic definitions related to the collapse
operator \cite{Whi39}. \\
Let $K$ be a complex and let $\sig,\tau$ be two distinct faces of~$K$.
The couple $(\sig,\tau)$ is a {\em free pair for $K$}
if $\tau$ is the only face of $K$ that contains
$\sig$.
Thus, the face $\tau$ is necessarily a facet of $K$.
If $(\sig,\tau)$ is a free pair for $K$, then $L = K \setminus \{\sig,\tau \}$
is {\em an elementary collapse of $K$}, and $K$ is {\em an elementary expansion of $L$}.
We say that
$K$ {\em collapses onto $L$},
or that $L$ {\em expands onto $K$},
if there exists a sequence
$\langle K = M_0,...,M_k = L \rangle$ such that
$M_i$ is an elementary collapse of $M_{i-1}$, $i \in [1,k]$. The complex
$K$ is {\em collapsible} if $K$ collapses onto a \emph{vertex}, that is, onto a complex of the form
$\{ \{ a \} \}$.
We say that $K$ is {\em (simply) homotopic to $L$},
or that $K$ and $L$ are {\em (simply) homotopic},
if there
exists a sequence  $\langle K = M_0,...,M_k=L \rangle$ such that
$M_i$ is an elementary collapse or an elementary expansion of $M_{i-1}$,
$i \in [1,k]$.
The complex $K$ is {\em (simply) contractible} if $K$ is simply homotopic
to a vertex.

\section{Morse sequences}
\label{sec:sequences}

Let us  start first with the definition of perforations and fillings.

Let $K,L$ be simplicial complexes.
If $\sigma \in K$ is a facet of $K$ and if $L = K \setminus \{\sigma \}$, we say that
$L$ is {\em an elementary perforation of $K$}, and that
$K$ is {\em an elementary filling of $L$}.

These transformations were introduced by Whitehead in a seminal paper \cite{Whi39}. Combined with collapses and expansions, it has been shown that we obtain four operators
that correspond to the homotopy equivalence between two simplicial complexes  (Th. 17 of \cite{Whi39}).
See also \cite{Ber21} which provides another kind of equivalence based on a variant of these operators.

In this paper, we introduce the notion of a ``Morse sequence'' by simply considering expansions and fillings
of a simplicial complex.

\begin{definition} \label{def:seq1}
Let $K$ be a simplicial complex. A \emph{Morse sequence (on $K$)} is a sequence
$\ms = \langle \emptyset = K_0,...,K_k =K \rangle$ of simplicial complexes such that,
for each $i \in [1,k]$, $K_i$ is either an elementary expansion or an elementary filling of $K_{i-1}$.
\end{definition}

Let $\ms = \langle K_0,...,K_k \rangle$ be a Morse sequence.
For each $i \in [1,k]$: \\
- If $K_i$ is an elementary filling of $K_{i-1}$, we write $\hat{\sigma}_i$  for the simplex $\sigma$ such that $K_i = K_{i-1} \cup \{\sigma\}$, we say that
the face $\sigma$ is \emph{critical for $\ms$}. \\
- If $K_i$ is an elementary expansion of $K_{i-1}$, we write
$\hat{\sigma}_i$ for the free pair $(\sigma,\tau)$ such that $K_i = K_{i-1} \cup \{\sigma,\tau \}$, we say that
$\hat{\sigma}_i$, $\sigma$, $\tau$, are \emph{regular for $\ms$}. \\
We write $\mss = \langle \hat{\sigma}_1,..., \hat{\sigma}_k\rangle$, and we say that
$\mss$ is a \emph{(simplex-wise) Morse sequence}.
Clearly, $\ms$ and $\mss$ are two equivalent forms.
We shall pass from one of these forms to the other without
notice.

\begin{figure*}[tb]
    \centering
    \begin{subfigure}[t]{0.31\textwidth}
        \centering
        \includegraphics[height=.99\textwidth]{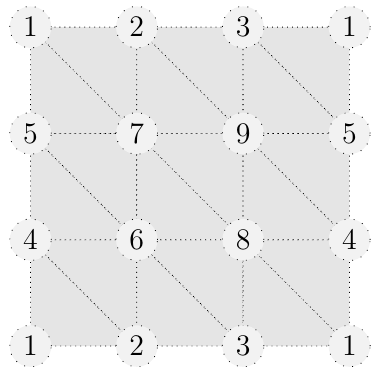}
        \caption{}
    \end{subfigure}%
    ~
    \begin{subfigure}[t]{0.31\textwidth}
        \centering
        \includegraphics[height=.99\textwidth]{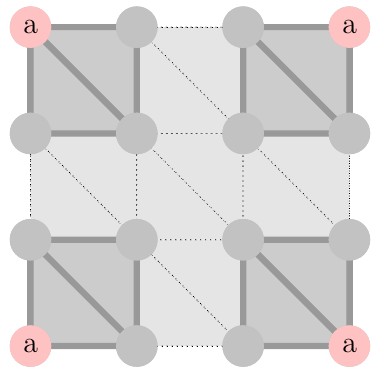}
        \caption{}
    \end{subfigure}%
    ~
    \begin{subfigure}[t]{0.31\textwidth}
        \centering
        \includegraphics[height=.99\textwidth]{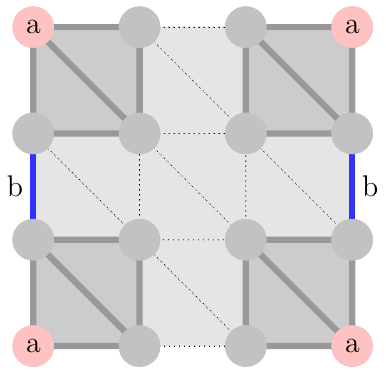}
        \caption{}
    \end{subfigure}%
    
    \begin{subfigure}[t]{0.31\textwidth}
        \centering
        \includegraphics[height=.99\textwidth]{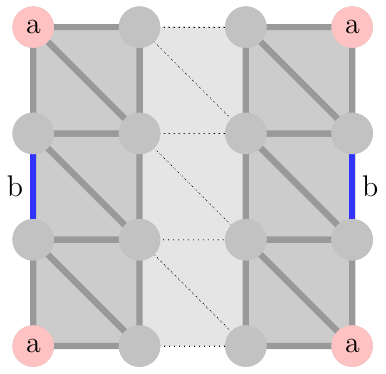}
        \caption{}
    \end{subfigure}%
    ~
    \begin{subfigure}[t]{0.31\textwidth}
        \centering
        \includegraphics[height=.99\textwidth]{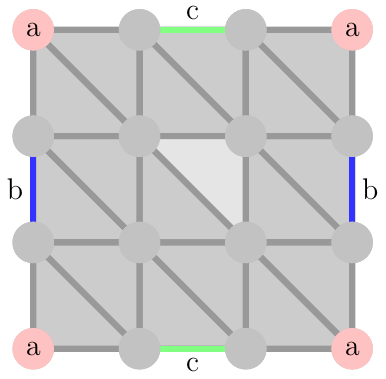}
        \caption{}
    \end{subfigure}
    ~
    \begin{subfigure}[t]{0.31\textwidth}
        \centering
        \includegraphics[height=.99\textwidth]{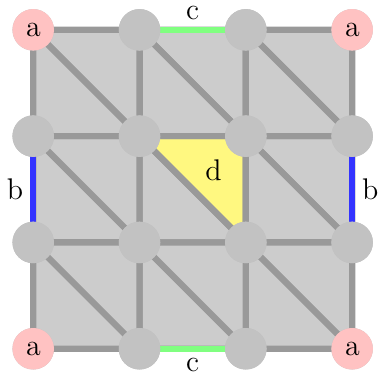}
        \caption{}
    \end{subfigure}

 \caption{A Morse sequence on the torus. (a) A triangulation, points with the same label are identified. (b) The sequence begins with the critical 0-simplex~$a$.  Elementary expansions are added to the sequence
 until we obtain a maximal expansion from $a$. (c)  The critical 1-simplex $b$ is added to the sequence. (d) A maximal expansion from $b$ is done. (e)
 The second critical 1-simplex $c$ is added, and a maximal expansion from $c$ is done. (f) The critical 2-simplex~$d$ is added. }
 \label{fig:MorseSequenceTorus}
\end{figure*}

Observe that, if $\ms = \langle K_0,...,K_k \rangle$ is a Morse sequence, with $k \geq 1$, then
$K_1$ is necessarily a filling of $\emptyset$. 
Thus, $K_1$ is necessarily a vertex. That is, $K_1$ is made of a single $0$-simplex that is critical for $\ms$.

Fig.~\ref{fig:MorseSequenceTorus} presents an example of a Morse sequence $\ms$ on a torus $T$. There are different ways to obtain a Morse sequence. 
In Fig.~\ref{fig:MorseSequenceTorus}, we apply the following strategy.
We build $\ms$ from the left to the right. Starting from~$\emptyset$, we obtain $T$ by iterative elementary expansions and fillings.
Also, we make maximal expansions, that is, 
we make a filling only if no elementary expansion can be made. \\

\begin{remark}
Let $\mss = \langle \hat{\sigma}_1,..., \hat{\sigma}_k\rangle$ be a Morse sequence and let $\hat{\sigma}_i$, $\hat{\sigma}_j$, $j >i$, be two consecutive
critical faces of $\mss$, that is, $\hat{\sigma}_{i+1},...,\hat{\sigma}_{j-1}$ are regular pairs.
Then, as a direct consequence of the definition of a Morse sequence, the complex $X_{j-1}$ collapses onto $X_{i}$.
This property is the core of a fundamental theorem, called \emph{the collapse theorem}, which makes the link between the basic definitions of  discrete Morse theory
and discrete homotopy (See Theorem 3.3 of \cite{For98a} and Theorem 4.27 of \cite{Sco19}).
In a certain sense, we can say that Morse sequences provide an introduction to discrete Morse theory by starting from this property.
\end{remark}

\begin{remark}
Any Morse sequence $\ms$ on $K$ is a \emph{filtration on $K$}, that is a sequence of nested complexes  $\langle \emptyset = K_0,...,K_k =K \rangle$
such that, for $i \in [0,k-1]$, we have $K_i \subseteq K_{i+1}$; see \cite{DW22}.
Also any \emph{simplex-wise filtration on $K$} is a special case of a Morse sequence where,
for $i \in [0,k-1]$, $K_{i+1} \setminus K_i$ is made of a single simplex.
That is, a simplex-wise filtration is a Morse sequence which is made solely of fillings; all faces of $K$ are critical for such a sequence.
\end{remark}

\section{Optimal and perfect Morse sequences}
\label{sec:optimal}

In the next two sections (Sections 5 and 6), we will see that a Morse sequence is an alternative way to represent the
gradient vector field of any arbitrary discrete Morse function.
Thus, we may directly transpose, without loss of generality, some notions relative to Morse functions to Morse sequences.
In the following, we give an illustration of such a transposition for the notions of optimal and perfect discrete Morse functions (see Def. 2.87 and Def. 4.6 of \cite{Sco19}).
Also, we give an exemple of a classical result that may be proved directly thanks to the notion of a Morse sequence (Proposition \ref{pro:seq01}).

Let $\ms$ be a Morse sequence on a complex $K$.
We say that $\ms$ is \emph{optimal} if the number $N$ of critical faces
for $\ms$ is minimal. That is, the number of faces that are critical for any other Morse sequence on $K$ is greater or equal
to $N$. \\
If $dim(K) = d$, the \emph{Morse vector of $\ms$} is the vector $\vec{c}(\ms) = (c_0,\dots,c _p, \dots,c_d)$ where
$c_p$ is the number of $p$-simplexes that are critical for $\ms$.
We denote by $\vec{b}(K)$ the vector $\vec{b}(K) = (b_0,\dots,b_p,\dots, b_d)$ where
$b_p$ is the $pth$ Betti number (mod. 2) of $K$
(see \cite{Gib10}). We also use the notations
$c_p(\ms)$ and $b_p(K)$ when $\ms$ and $K$ are not clear from the context.

We say that a Morse sequence $\ms$ on $K$ is \emph{perfect} if $\vec{c}(\ms) = \vec{b}(K)$. 
In other words, a Morse sequence $\ms$ on $K$ is perfect if each number $b_p$ of ``$p$-dimensional holes of $K$'' is equal to 
the number $c_p$ of critical $p$-simplexes of $\ms$.

Suppose a complex $K$ is collapsible. Then we have $\vec{b}(K) = (1,0,\dots,0)$. 
Also, we easily see that $K$ admits a Morse sequence which has a single critical face, this face being a $0$-simplex.
For this sequence, we have $\vec{c}(\ms) = (1,0,\dots,0)$, thus $K$ admits a
perfect discrete Morse sequence.

\begin{figure*}[tb]
    \centering
    \begin{subfigure}[t]{0.32\textwidth}
        \centering
        \includegraphics[width=.99\textwidth]{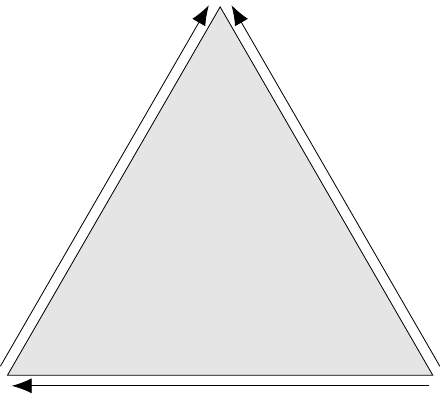}
        \caption{}
    \end{subfigure}%
    ~
    \begin{subfigure}[t]{0.32\textwidth}
        \centering
        \includegraphics[width=.99\textwidth]{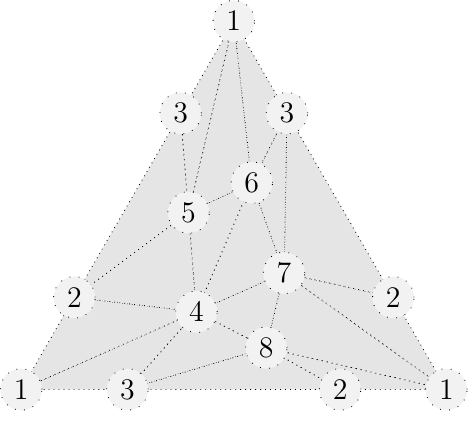}
        \caption{}
    \end{subfigure}%
    ~
    \begin{subfigure}[t]{0.32\textwidth}
        \centering
        \includegraphics[width=.99\textwidth]{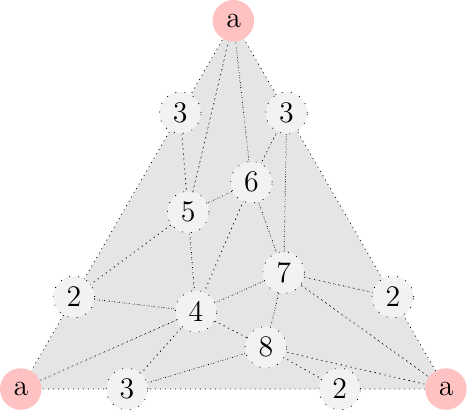}
        \caption{}
    \end{subfigure}%

    \begin{subfigure}[t]{0.32\textwidth}
        \centering
        \includegraphics[width=.99\textwidth]{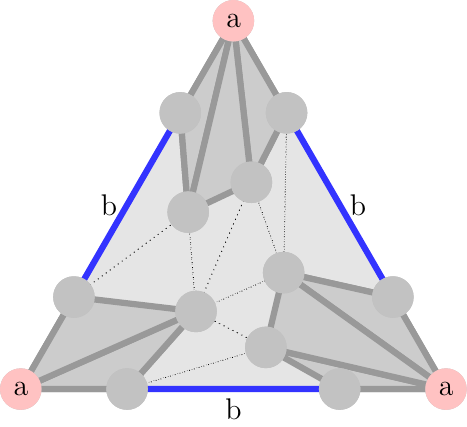}
        \caption{}
    \end{subfigure}%
    ~
    \begin{subfigure}[t]{0.32\textwidth}
        \centering
        \includegraphics[width=.99\textwidth]{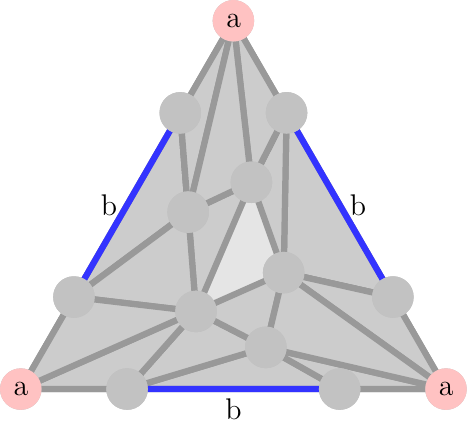}
        \caption{}
    \end{subfigure}
     ~
    \begin{subfigure}[t]{0.32\textwidth}
        \centering
        \includegraphics[width=.99\textwidth]{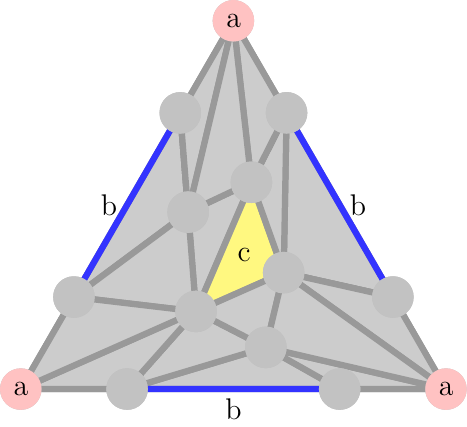}
        \caption{}
    \end{subfigure}

 \caption{A Morse sequence on the dunce hat.  (a) the dunce hat, the three edges of the triangle have to be identified with the arrows. (b) A triangulation of the dunce hat. (c) The sequence begins with the critical 0-simplex $a$. 
 c) A maximal expansion from $a$ is done, then the 1-critical simplex $b$ is added. (e) A maximal expansion from $b$. (f) The critical 2-simplex
 $c$ is added.}
  \label{fig:MorseSequenceDunceHat1}
\end{figure*}

Now, let us consider a complex that is contractible but not collapsible.
The dunce hat \cite{Zee64} is a basic example of such a complex.
In Fig.~\ref{fig:MorseSequenceDunceHat1}, a Morse sequence~$\ms$ for a triangulation $D$ of the dunce hat is given; the same strategy as above has been used.
We see that, in this example, we have $\vec{c}(\ms) = (1,1,1)$. But, by contractibility of $D$, we have
$\vec{b}(D) = (1,0,0)$. This leads to the question: Is it possible to have a perfect Morse sequence for $D$? 

We have the answer to this question by simply reading the definition of a Morse sequence:
If  a sequence $\ms$ on $K$ has a single critical simplex (therefore, a $0$-simplex), then
clearly the complex $K$ is collapsible. 

Thus -- thanks to the notion of a Morse sequence -- we have
a straightforward proof of the following classical result (see Prop. 4.10 of \cite{Sco19}, see also \cite{Aya12} and \cite{Aya12a}).

\begin{proposition} \label{pro:seq01}
   Let $K$ be a complex with $\vec{b}(K) = (1,0,\cdots,0)$. 
   The complex $K$ admits a
perfect discrete Morse sequence if and only if $K$ is collapsible.
\end{proposition}

To conclude this section, we underline a fundamental link  between Morse sequences and homology.
Let $K = L \cup \{\sig\}$ be an elementary filling of $L$, with $\sig \in K^{(p)}$, $p \geq 1$. It is well-known that
the addition of $\sig$  will either increase $b_p(L)$ by 1 or
decrease $b_{p-1}(L)$ by 1 (but not both), all other Betti numbers being unaffected; see Lemma 3.36 of \cite{Sco19}.
Also, it is well-known that, if $K$ is an elementary expansion of $L$, then all Betti numbers are unaffected.
This leads us to the following definition where each critical simplex is either positive or negative.

\begin{definition} \label{def:seq3}
Let $\ms = \langle K_0,...,K_k  \rangle$ be a Morse sequence and $\mss = \langle \hat{\sigma}_1,..., \hat{\sigma}_k\rangle$.
Let $\hat{\sig}_i = \sig_i$ be a critical $p$-simplex for $\ms$.
We say that $\sig_i$ is \emph{positive for $\ms$} if $i = 1$ or if $b_p(K_i) = b_p(K_{i-1}) +1$. We say that 
$\sig_i$ is \emph{negative for $\ms$}
if $i \geq 2$ and if $b_{p-1}(K_i) = b_{p-1}(K_{i-1}) - 1$, with $p \geq 1$. 
\end{definition}

We check at once that a Morse sequence $\ms$ is perfect if and only if all critical simplexes 
for $\ms$ are positive for $\ms$.
For example: \\
- A perfect Morse sequence $\ms$ on a collapsible complex is made of a single critical simplex that is positive for $\ms$
(see Proposition~\ref{pro:seq01}). \\
-  The simplexes $a$, $b$, $c$, and $d$ are positive for the Morse sequence of the torus given Fig.~\ref{fig:MorseSequenceTorus}. \\
- The simplexes $a$ and $b$ are positive for the Morse sequence 
of the dunce hat given Fig.~\ref{fig:MorseSequenceDunceHat1}; the simplex $c$ is negative for this sequence. 

We observe also that, from the above, we deduce immediately  the following classical inequality:
each Betti number $b_p$ is lower or equal to the number $c_p$ of $p$-simplexes that are critical  
for $\ms$ (see Th.~4.1 of \cite{Sco19}). 
It follows that a Morse sequence is necessarily optimal whenever it is perfect.

\section{Discrete vector fields and gradient paths}
\label{sec:gradient}

From the definition of a Morse sequence, we can 
immediately derive the following notion of a gradient vector field.

\begin{definition} \label{def:seq2}
The \emph{gradient vector field of a Morse sequence
$\ms$} is the set of all regular pairs for $\ms$.
We say that two Morse sequences $\ms$ and $\overrightarrow{V}$ on
a given complex $K$ are \emph{equivalent} if
they have the same gradient vector field.
\end{definition}

Let us recall the definitions of a discrete vector field and a $p$-gradient path, see Definitions 2.43 and 2.46 of \cite{Sco19}.

Let $K$ be a complex and $V$ be a set of pairs $(\sigma,\tau)$, with $\sigma, \tau \in K$ and
$\sigma \in \partial \tau$.
We say that $V$ is a  \emph{(discrete) vector field on $K$}
if each simplex of $K$ is in at most one pair of $V$.
We say that $\sigma \in K$ is \emph{critical for $V$} if $\sig$ is not in a pair of $V$.

Let $V$ be a vector field on a complex $K$.
A \emph{$p$-gradient path in $V$ (from $\sig_0$ to $\sig_k$)}
is a sequence
$\pi = \langle \sig_0,\tau_0,\sig_1,\tau_1,...,\sig_{k-1},\tau_{k-1}, \sigma_{k}\rangle$,
with $k \geq 0$, composed of faces $\sig_i \in K^{(p)}$, $\tau_i \in K^{(p+1)}$
such that, for all $i \in [0,k-1]$,
$(\sig_i,\tau_{i})$ is in~$V$,
$\sigma_{i+1} \subset \tau_i$, and $\sig_{i+1} \not= \sig_i$.
This sequence $\pi$ is said to be \emph{trivial} if $k =0$, that is, if 
$\pi = \langle \sig_0 \rangle$; otherwise, if $k \geq 1$, we say that $\pi$ is \emph{non-trivial}.
Also, the sequence $\pi$ is \emph{closed} if $\sig_0 = \sig_k$.
We say that a vector field $V$ is \emph{acyclic} if $V$ contains
no non-trivial closed $p$-gradient path. 

Now, let $\ms$ be a Morse sequence on $K$. Then, the gradient vector field of $\ms$ is clearly a vector field.
We say that a $p$-gradient path in this vector field is a \emph{$p$-gradient path in $\ms$}. 

In the sequel of this section,
we show that a Morse sequence may be seen as an alternative way to represent the
gradient vector field of an arbitrary discrete Morse function. A classical result of discrete Morse t\textit{}heory states that
a discrete vector field $V$ is the gradient vector field of a
discrete Morse function if and only if $V$ is acyclic (Theorem 2.51 of \cite{Sco19}). Thus, in order to achieve this goal,
we establish the equivalence between gradient vector fields of Morse sequences and acyclic vector fields (Theorem \ref{th:DVF}).
Before, we introduce the notion of a maximal $p$-gradient path. Such a path allows us to extract, in the top dimension of a complex $K$,
either a critical simplex or a free pair for $K$ (Lemma \ref{lem:DVF1}). This formalizes  an incremental
deconstruction of the complex, which is usually 
given with certain Morse functions, see Remark 13 of \cite{Ben17}.

Let $V$ be a vector field on $K$ and 
let $\pi = \langle \sig_0,\tau_0,...,\sig_{k-1},\tau_{k-1}, \sigma_{k} \rangle$ be a $p$-gradient path in $V$.
We say that a pair of simplexes $(\eta,\nu)$ is an \emph{extension of~$\pi$ (in $V$)}
if $\langle \eta,\nu,\sig_0,\tau_0,...,\sig_{k-1},\tau_{k-1}, \sigma_{k} \rangle$ or if $\langle \sig_0,\tau_0,...,\sig_{k-1},\tau_{k-1}, \sigma_{k}, \nu, \eta \rangle$ is a $p$-gradient path in $V$.
We say that $\pi$ is \emph{maximal (in $V$)} if $\pi$ has  no extension in~$V$. 
If $V$ is acyclic, it can be checked that, for any 
$p \geq 0$, there exists a maximal $p$-gradient path in $V$.
To see this point, we can pick an arbitrary (possibly trivial) $p$-gradient path and 
extend it iteratively with extensions. If $V$ is acyclic, we obtain 
a maximal path after a finite number of extensions.

 \begin{lemma}[deconstruction] \label{lem:DVF1}
Let $V$ be an acyclic vector field on $K$, with $dim(K) = d$.
Then, at least one of the following holds: \\
1) There exists a facet $\tau$ of $K$, with $dim(\tau) = d$, that is critical for $V$. \\
2) There exists a pair $(\sig,\tau)$ in $V$, with $dim(\tau) = d$, that is a free pair for $K$.
\end{lemma}

\begin{proof}
If $K$ has a $d$-face that is critical for $V$, then we are done. 
Suppose there is no such faces in $K$. If $d= 0$, then all the $0$-faces of $K$ are critical, thus we must have $d \geq 1$.
Let $\tau$ be an arbitrary $d$-face of $K$.
Since $\tau$ is not critical, 
there exists a pair $(\sig,\tau)$ that is in $V$.
Since $d \geq 1$, there is a face $\sig' \in K$
such that $\pi' = \langle \sig,\tau,\sig' \rangle$ is a $(d-1)$ gradient path in~$V$.
By iteratively extending $\pi'$ with extensions we obtain 
a maximal $(d-1)$-gradient path in $V$ that is non-trivial.
Let $\pi = \langle \sig_0,\tau_0,...,\sig_{k-1},\tau_{k-1}, \sigma_{k} \rangle$ be such a path, we have $k \geq 1$.
If $(\sigma_0,\tau_0)$ is a free pair for $K$, then we are done.
Otherwise, $\sig_0$ must be a subset of a $d$-simplex  $\nu$, with $\nu \not= \tau_0$. By our hypothesis $\nu$ is not critical for $V$. Since $\nu$ is a facet
for $K$, there must exist a $(d-1)$-simplex $\eta$, $\eta \not= \sig_0$, such that $(\eta,\nu)$ is in $V$.
In this case, the path $\pi' = \langle \eta, \nu, \sigma_0, \tau_0,...,\sig_{k-1},\tau_{k-1}, \sigma_{k} \rangle$ would be a $(d-1)$-gradient path in~$V$.
Thus, the path $\pi$ would not be maximal, a contradiction:
the pair $(\sigma_0,\tau_0)$ must be a free pair for $K$. \qed 
\end{proof}

\begin{theorem} \label{th:DVF}
Let $K$ be a simplicial complex. A vector field $V$ on $K$ is acyclic if and only if $V$ is the gradient vector field of a
Morse sequence on $K$.
\end{theorem}

\begin{proof}
i) Let $\ms = \langle \emptyset = K_0,...,K_i,...,K_l = K  \rangle$ be a Morse sequence on $K$ and let $V$ be the gradient vector field of $\ms$.
For each $\sig \in K$, let $\rho(\sig)$ be the index $i$ such that 
$\sig \in K_i$ and $\sig \not\in K_{i-1}$. 
Now, let $\pi = \langle \sig_0,\tau_0,\sig_1,\tau_1,...,\sig_{k-1},\tau_{k-1}, \sigma_{k}\rangle$, $k \geq 1$, be a non-trivial $p$-gradient path in $V$. 
For all $i \in [0,k-1]$, $(\sig_i,\tau_{i})$ is in~$V$, thus 
$\rho(\sig_i) = \rho(\tau_i)$. Since $\sigma_{i+1} \subset \tau_i$ and since 
$\ms$ is a filtration, we have $\rho(\sig_{i+1}) \leq \rho(\tau_i)$. 
Since $\sig_{i+1} \not= \sig_i$ the pair $(\sig_{i+1},\tau_i)$ is not a regular pair for 
$\ms$, thus we have $\rho(\sig_{i+1}) < \rho(\tau_i)$.
It follows that, for all $i \in [0,k-1]$, we have $\rho(\sig_{i+1}) < \rho(\sig_i)$.
This gives $\rho(\sig_{k}) < \rho(\sig_0)$. It means that $\sig_{k} \not= \sig_{0}$,
and so the path $\pi$ cannot be closed. Consequently the vector field $V$ is acyclic. \\
ii) 
Let $V$ be an acyclic vector field on $K$, with $dim(K) = d$. \\
1) Suppose there exists a facet $\tau$ of $K$, with $dim(\tau) = d$, that is critical for $V$. \\
Let $K' = K \setminus \{\tau\}$ and $V' = V$.
Then, the set $V'$ is also an acyclic vector field on the complex $K'$. \\
2) Suppose there exists a pair $(\sig,\tau)$ in $V$, with $dim(\tau) = d$, that is a free pair for $K$.
Clearly, the set $V' = V \setminus \{(\sigma, \tau) \}$ is also an acyclic vector field on the complex
$K' = K \setminus \{\sigma, \tau \}$.\\
By 1), 2), and by Lemma \ref{lem:DVF1}, we can build inductively two sequences $\overleftarrow{W}= \langle K = K_0,...,K_k = \emptyset \rangle$
and $\langle V = V_0,...,V_k = \emptyset \rangle$ such that, for each $i \in [1,k]$: \\
- either $K_i$ is an elementary perforation of $K_{i-1}$ and $V_i = V_{i-1}$, \\
- or $K_i = K_{i-1} \setminus \{\sigma, \tau \}$ is an elementary collapse of $K_{i-1}$ and $V_i = V_{i-1} \setminus \{(\sigma, \tau) \}$.\\
By considering the inverse of $\overleftarrow{W}$ we obtain
the sequence $\ms = \langle K_k= K'_0,...,K'_k = K_0 \rangle$, which is such that, for each $i \in [1,k]$,
either $K'_i$ is an elementary expansion of $K'_{i-1}$, or $K'_i$ is an elementary filling of $K'_{i-1}$.
In other words, $\ms$ is a Morse sequence on $K_0 = K$; the gradient field of $\ms$ is precisely $V,$ as required. \qed
\end{proof}

\section{Morse functions and Morse sequences}
\label{sec:functions}

Discrete Morse theory is classically introduced through the concept of a discrete Morse function. In this section we 
show that it is possible, in a straightforward manner,  to make a link between Morse sequences and these functions. 

We first introduce the notion of a Morse function on a 
Morse sequence $\ms$.

\begin{definition}
Let $\ms$ be a Morse sequence on $K$ and
$\mss = \langle \hat{\sigma}_1,\ldots, \hat{\sigma}_k\rangle$.
A map $f \colon K \to \bb{Z}$ is a \emph{Morse function on $\ms$}  whenever $f$ satisfies the two conditions:

1) If $\hat{\sigma}_i = \sig_i$ is critical for $\ms$ and $\sig \in \partial(\sig_i)$, then $f(\sig_i) > f(\sig)$.

2) If $\hat{\sigma}_i = (\sig_i,\tau_i)$ is regular for $\ms$, then $f(\sig_i) \geq f(\tau_i)$.     
\end{definition}

Now, we can check that the following definition of a Morse function on a simplicial complex $K$ is equivalent to the classical one \cite{For98,For02}.

Let $K$ be a simplicial complex and let $f \colon K \to \bb{Z}$ be a map on $K$. Let $V$ be the set of all pairs $(\sigma, \tau)$,
with $\sig, \tau \in K$, such that $\sig \in \partial(\tau)$ and $f(\sig) \geq f(\tau)$.
If each $\nu \in K$ is in at most one pair in $V$,
we say that $f$ is  a \emph{Morse function  on~$K$}, and $V$ is the \emph{gradient vector field of $f$}.
We say that two Morse functions on $K$ are \emph{equivalent} if they have the same gradient vector field.

Let $f$ be  a Morse function  on $K$, and $V$ be the gradient vector field of $f$.
From the above definition, the set $V$ is a discrete vector field on $K$.
If $\pi = \langle \sig_0,\tau_0,\sig_1,\tau_1,...,\sig_{k-1},\tau_{k-1}, \sigma_{k}\rangle$ is a $p$-gradient path in $V$,
we see that we must have $f(\sig_i) \geq f(\tau_i)$, and also $f(\tau_i) > f(\sig_{i+1})$.
Thus, $f(\sig_0) > f(\sig_k)$ whenever $k \geq 1$. It means that $V$ contains
no non-trivial closed $p$-gradient path. In other words, we have the classical result:

 \begin{proposition} \label{pro:mf1}
If $f$ is a Morse function on $K$, then the gradient vector field of $f$ 
is an acyclic vector field.
 \end{proposition}

Let $\ms$ be a Morse sequence on $K$. 
We see that a Morse function on $\ms$ is indeed a Morse function
on $K$, the gradient vector field of this Morse function is
precisely the gradient vector field of $\ms$.
Conversely, by Proposition \ref{pro:mf1} and by Theorem~\ref{th:DVF}, if $f$ is a Morse function on $K$, then there exists a
Morse sequence $\ms$ on $K$ which has the same gradient vector field as $f$. It is easy to check that~$f$ is also a Morse function 
on $\ms$. This leads us to the following result.

 \begin{theorem} \label{th:mf3}
If $f$ is a Morse function on $K$, then there exists a Morse sequence~$\ms$ 
on $K$ such that $f$ is a Morse function on $\ms$.
Furthermore, any Morse function on $\ms$ is equivalent to $f$.
 \end{theorem}

We introduce hereafter a particular kind of Morse 
function. 
Since a Morse sequence is a filtration, the following
function $f$ is indeed a Morse function on~$\ms$.

\begin{definition}
Let $\ms$ be a Morse sequence on $K$ and
$\mss = \langle \hat{\sigma}_1,\ldots, \hat{\sigma}_k\rangle$.
The \emph{canonical Morse function of $\ms$}  is the map 
$f \colon K \to \bb{Z}$ such that:

1) If $\hat{\sigma}_i = \sig_i$
is critical for $\ms$, then $f(\sig_i) =i$.

2) If $\hat{\sigma}_i = (\sig_i,\tau_i)$
is regular for $\ms$, then $f(\sig_i) = f(\tau_i) =i$. 
\end{definition}

As a consequence of Theorem \ref{th:mf3}, any Morse function 
on $K$ is equivalent to a canonical Morse function.

We note that a canonical Morse function $f$ is \emph{flat}, that is,
we have  $f(\sig) = f(\tau)$ whenever $(\sig,\tau)$ is in the gradient vector field of $f$
(Definition 4.14 of \cite{Sco19}).
Also~$f$ is \emph{excellent}, that is, all values of the critical simplexes
are distinct (Definition  2.31 of \cite{Sco19}).
In fact, a canonical Morse function has the three properties
which define a basic Morse function (see \cite{Ben16} and also Definition 2.3 of \cite{Sco19}).

Let $f \colon K \to \bb{Z}$ be a map on $K$. 
We say that $f$ is a \emph{basic Morse function}
if $f$ satisfies the properties:

1) \emph{monotonicity}: we have $f(\sig) \leq f(\tau)$ whenever $\sig \subseteq \tau$;

2) \emph{semi-injectivity}: for each $i \in \bb{Z}$, the cardinality of $f^{-1}(i)$ is at most 2;

3) \emph{genericity}: if $f(\sig) = f(\tau)$, then either  $\sig \subseteq \tau$ or $\tau \subseteq \sig$. \\
We observe that, if $f$ is a basic Morse function on $K$,
then we can build a Morse sequence $\ms$
if we pick the simplexes of $K$ by increasing values of $f$. 
For each~$i$, $f^{-1}(i)$ gives a critical simplex if 
the cardinality of $f^{-1}(i)$ is one, and
$f^{-1}(i)$ gives a regular pair if 
the cardinality of $f^{-1}(i)$ is two.

Let $f$ and $g$  be two basic Morse functions on $K$. 
We say that $f$ and $g$ are \emph{strongly equivalent}
if $f$ and $g$ induce the same order on $K$. That is, we have 
$f(\sig) \leq f(\tau)$ if and only if 
$g(\sig) \leq g(\tau)$.

With the above scheme for building a Morse sequence from 
a basic Morse function, we obtain the following result.

\begin{proposition}
Let $f$ be a basic Morse function on $K$.
There exists one and only one Morse sequence $\ms$ such that 
the canonical Morse function of $\ms$ is strongly equivalent to $f$.
\end{proposition}

\section{Maximal Morse sequences}
\label{sec:maximal}

Building a gradient vector field from a complex is a fundamental issue in discrete Morse theory.
This problem is equivalent to building a Morse sequence $\ms$ from a complex $K$.
Clearly, the two following schemes are two basic ways to achieve this goal: \\
1) \emph{The increasing scheme}. We build $\ms$ from the left to the right. Starting from~$\emptyset$, we obtain $K$ by iterative expansions and fillings.
We say that this scheme is \emph{maximal} if
we make a filling only if no expansion can be made. \\
2) \emph{The decreasing scheme}. We build $\ms$ from the right to the left. Starting from $K$, we obtain $\emptyset$ by iterative collapses and perforations.
We say that this scheme is \emph{maximal} if we make a perforation only if no collapse can be made. 

Clearly, any Morse sequence may be obtained by an increasing scheme and any Morse sequence may be obtained by a decreasing scheme.
By Theorem \ref{th:DVF}, it means that an arbitrary acyclic vector field may be obtained by each of these two schemes. 

Now, let us focus our attention on maximal increasing and maximal decreasing schemes. The purpose of these two schemes
is to try to minimize the number of critical
simplexes. Thus, a filling or a perforation is made
only if there is no other choice.
The examples given in Fig.~\ref{fig:MorseSequenceTorus}
and Fig.~\ref{fig:MorseSequenceDunceHat1} are instances of a maximal increasing scheme.

First, it can be seen that  the algorithm \emph{Random Discrete Morse},
proposed by Benedetti and Lutz in~\cite{Ben14}, corresponds exactly to a maximal decreasing scheme. 
See this paper for many details of the algorithm (computational complexity, implementation in GAP, comparison with other algorithms, lower bounds for discrete Morse vectors...). See also Section 2.3 and 
Algorithm 1 in~\cite{Sco19}.

Also, there is a link between a maximal increasing scheme and \emph{coreduction based algorithms}
\cite{Mro09,Har14,fugacci2019computing}. As mentioned in \cite{fugacci2019computing}, 
a coreduction is not feasible on a simplicial complex.
In fact, the coreduction algorithm presented in \cite{fugacci2019computing} may be formalized with a Morse sequence through the following definition.

\begin{definition} \label{def:cored}
Let $K$ be a simplicial complex. A \emph{coreduction  sequence (on $K$)} is a sequence
$\overrightarrow{C} = \langle K = C_0,..., C_k = \emptyset \rangle$ such that the sequence 
$\ms = \langle \emptyset = K_0 = K \setminus C_0,...,K_k = K \setminus C_k = K \rangle$ is a Morse sequence.
\end{definition}

In other words, a sequence $\overrightarrow{C} = \langle K = C_0,..., C_k = \emptyset \rangle$ is a coreduction sequence if, for each $i \in [1,k]$, $K \setminus C_i$
is either an elementary expansion or an elementary filling  of $K \setminus C_{i-1}$. It may be checked 
that the notion of a coreduction
presented in \cite{fugacci2019computing} fully agrees
with the above definition. It follows that the
corresponding maximal coreduction algorithm may be seen
as a maximal increasing scheme if we 
simply build a filtration with the simplexes that are removed by such an algorithm; see Section 5  of \cite{fugacci2019computing}.

Thus, Morse sequences allow us to retrieve two methods 
for building a gradient vector field,
which try to minimize the number of
critical simplexes. Equivalently, they try to find optimal Morse sequences. It is worth mentioning that this problem is, in general, NP-hard \cite{Jos06}.
Therefore, these methods do not, in general, give optimal results. 

Now, let $\ms = \langle \emptyset = K_0,..., K_k = K \rangle$
be a Morse sequence on $K$.
We write $\overleftarrow{W}$ for the inverse of the sequence $\ms$, that is, we have
$\overleftarrow{W} = \langle K = K_k,..., K_0 = \emptyset \rangle$.

Thus, if $\ms$ is a Morse sequence, $\overleftarrow{W}$ is a sequence $\langle L_0,..., L_k \rangle$
such that, for each $i \in [1,k]$,
$L_i$ is either an elementary collapse or an elementary perforation of $L_{i-1}$.
The following definition is a formal presentation of maximal increasing and decreasing schemes.
See also \cite[Definition 11]{Ben17}  for an alternative formalization of a maximal decreasing scheme based 
on basic Morse functions. 

\begin{definition} \label{def:max1}
Let $\ms = \langle \emptyset = K_0,..., K_k = K \rangle$
be a Morse sequence on $K$.
For any $i \in [0,k]$, we say that $K_i$ is \emph{maximal for $\ms$} (resp. \emph{maximal for $\mso$}) if
no elementary expansion (resp. collapse) of $K_i$ is a subset of $K$.

\noindent
1) We say that $\ms$ is \emph{maximal} if, for any $i \in [1,k]$, the complex $X_{i-1}$ is maximal  for $\ms$ whenever $X_i$ is critical
for $\ms$.

\noindent
2)
We say that $\mso$ is \emph{maximal} if, for any $i \in [0,k-1]$, the complex $X_{i+1}$ is maximal  for $\mso$ whenever $X_i$ is critical
for $\mso$.
\end{definition}

Perhaps surprisingly, there exist some significant differences between these two schemes, in particular in regard
to the number of critical simplexes that are obtained.

The complex of Figure \ref{fig:lun}, already considered in \cite{Ben14} and \cite{fugacci2019computing}, 
illustrates this difference.
The complex $K$ in this example is a graph, that is, we have $dim(K) \leq 1$. 
In (a) and (b), the results that may be produced by a maximal decreasing scheme and by a maximal increasing scheme are given;
the corresponding Morse vectors are $(2,3)$ and $(1,2)$. This last vector  corresponds to the Betti numbers of $K$.
It can be seen that the result in (a) cannot be obtained by a maximal increasing scheme. 
Actually, the following result is easy to check.

 \begin{proposition} \label{pro:max1} Let $K$ be a complex, and let $\ms$ be a Morse sequence on $K$ that is maximal.
If $K$ is a graph, then $\ms$ is perfect.
 \end{proposition}

Now, let us consider the complex $K$ depicted in Figure \ref{fig:hachi}. This complex, given  by Hachimori in \cite{hachimori2000combinatorics,HachLib}, is a slight 
modification of the dunce hat.
We observe that $K$ is collapsible, therefore $\vec{b}(K) = (1,0,0)$.
We also observe that $K$ contains a single free pair, which is $(\{1,3\}, \{1,3,4\})$. Thus, any collapse sequence must begin with this pair.
Now, we see that we can build a spanning tree on $K$ that contains the edge $\{1,3\}$.
It is possible that a complex which is built in the first steps of a maximal increasing scheme of $K$ contains this edge. 
This edge will prevent further expansions of the sequence from recovering the full complex $K$.
Such a sequence $\ms$ is depicted in Figure \ref{fig:hachi}: $\ms$ is not perfect.
It can be seen that this cannot happen with a maximal decreasing scheme.
In fact, we have the following result.

 \begin{proposition} \label{pro:max2}
Let $K$ be a complex, with $dim(K) = 2$, and let $\ms$ be a Morse sequence on $K$ such that $\mso$ is maximal.
If $K$ is collapsible, then $\ms$ is perfect.
 \end{proposition}

\begin{figure}[tb]
    \centering
    \begin{subfigure}[t]{0.4\textwidth}
        \centering
        \includegraphics[width=.99\textwidth]{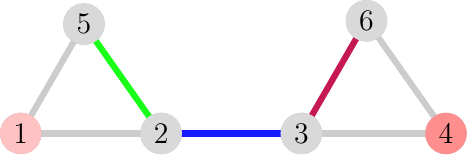}
        \caption{}
    \end{subfigure}%
    ~~~~~
    \begin{subfigure}[t]{0.4\textwidth}
        \centering
        \includegraphics[width=.99\textwidth]{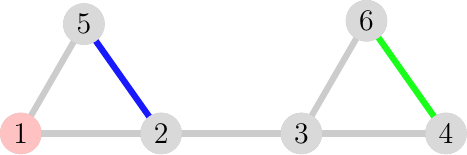}
        \caption{}
    \end{subfigure}%

 \caption{\label{fig:lun} A 1-dimensional complex. The results of a maximal decreasing scheme (a),  and
 a maximal increasing scheme (b). See text for details.
}

\end{figure}

\begin{figure}[tb]
    \centering
    ~
    \begin{subfigure}[t]{0.32\textwidth}
        \centering
        \includegraphics[width=.99\textwidth]{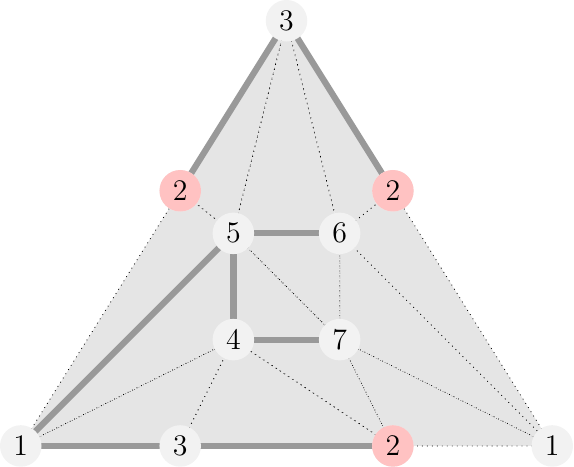}
        \caption{}
    \end{subfigure}%
    ~
    \begin{subfigure}[t]{0.32\textwidth}
        \centering
        \includegraphics[width=.99\textwidth]{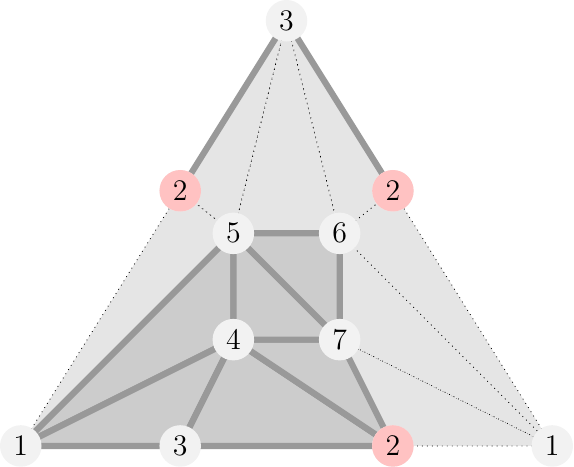}
        \caption{}
    \end{subfigure}%
    ~
    \begin{subfigure}[t]{0.32\textwidth}
        \centering
        \includegraphics[width=.99\textwidth]{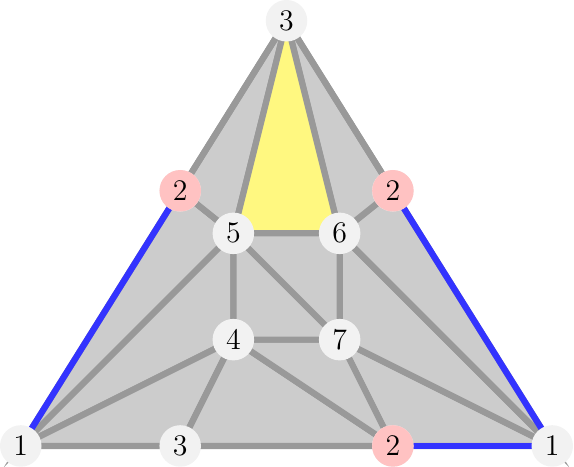}
        \caption{}
    \end{subfigure}%
\caption{\label{fig:hachi} A Morse sequence obtained by a maximal increasing scheme on Hachimori's example.
 (a): Starting from the critical 0-simplex 2, we create a maximal spanning tree that contains the edge $\{1,3\}$.
 (b): We then make all possible expansions. At this point, we have to select a critical 1-simplex, the edge $\{1,2\}$. 
 (c): We continue with expansions, until this is no longer possible. We then have to add the critical 2-simplex $\{3,5,6\}$. See text for a discussion.}
\end{figure}

\section{Conclusion}
\label{sec:conclusion}

In this paper, we introduce the notion of a Morse sequence for a simple presentation 
of some basic ingredients of discrete Morse theory: 
\begin{itemize}[noitemsep,topsep=0pt]
\item  The collapse theorem becomes a property that is contained in the very definition of a Morse sequence;
\item The link between Morse sequences and different kinds of Morse
functions is straightforward;
\item  A Morse sequence may represent the gradient vector field of an arbitrary
discrete Morse function; 
\item Maximal Morse sequences formalize two basic schemes
for building the gradient vector fields of an arbitrary simplicial complex. 
\end{itemize}

Morse sequences are not only interesting by themselves,
they also offer new perspectives for exploring the topology of simplicial complexes.
For example, adding information to Morse sequences leads to novel schemes for computing topological invariant such as cycles, cocycles, and Betti numbers. This can be achieved by defining {\em Morse frames}, which are maps that associate a set of critical simplexes to each simplex of the complex. 
See the companion paper~\cite{Bertrand2023MorseFrames}
where this approch is explored. 

\bibliographystyle{splncs04}
\bibliography{biblio}
\end{document}